\documentclass[letterpaper]{article}

\usepackage[preprint]{aaai2027}
\usepackage[hyphens]{url}
\usepackage{graphicx}
\usepackage{natbib}

\usepackage{caption}
\usepackage{amsmath,amssymb,amsthm,mathtools}
\usepackage{booktabs}
\usepackage{enumitem}
\newtheorem{theorem}{Theorem}
\newtheorem{proposition}{Proposition}
\newtheorem{corollary}{Corollary}
\newtheorem{lemma}{Lemma}
\newtheorem{assumption}{Assumption}
\newtheorem*{remark}{Remark}
\newcommand{\E}{\mathbb{E}}
\newcommand{\Var}{\operatorname{Var}}
\newcommand{\Cov}{\operatorname{Cov}}
\newcommand{\Corr}{\operatorname{Corr}}
\newcommand{\R}{\mathbb{R}}
\newcommand{\Ind}{\mathbf{1}}
\newcommand{\cN}{\mathcal{N}}
\newcommand{\cI}{\mathcal{I}}
\newcommand{\cF}{\mathcal{F}}

\newcommand{\PhiBar}{\overline{\Phi}}
\newcommand{\dd}{\,\mathrm{d}}

\title{The Label Defines the Timescale:\\Trait--State Limits of Temporal-Aggregate Learning}
\author{Xizhe Zhang}
\affiliations{
ORCID: 0000-0002-8684-4591 \quad \textbullet \quad zhangxizhe@gmail.com
}

\begin{document}
\maketitle
\begin{abstract}
Machine-learning benchmarks often pair a label that aggregates a long temporal horizon with input observed through one or a few short windows. Their apparent performance ceiling may therefore be an acquisition-protocol ceiling rather than a model-capacity ceiling. We study labels of the form $\Theta_{g,T}=T^{-1}\int_0^T g\{Z(t)\}\dd t$ when the latent Gaussian process contains both a stable individual trait and a correlated within-individual state. An exact protocol-conditioned Bayes-risk identity provides a common tool. First, we decompose label variance into an $O(1)$ trait component and an $O(T^{-1})$ state component, explaining why a snapshot can retain cross-sectional predictability while poorly tracking within-person change. Second, we derive task-dependent effective temporal spans: mean labels depend on the ordinary correlation time, whereas occupation-time labels depend on an entire spectrum of higher-order correlation times. Third, state-driven occupation-label variance is maximal when the stable trait lies at the threshold; window efficiency decays much more slowly away from that boundary. Under an equal segment budget, exact risks and Monte Carlo experiments show that repeated segments at one time rapidly saturate, whereas temporally dispersed observations continue to increase state explainability. The trait ceiling uses quantities available from ordinary test--retest data; only the state ceiling requires short-lag temporal calibration. The results distinguish architectural limits from protocol limits and show that the label---not duration or segment count alone---defines the relevant timescale.
\end{abstract}

\section{Introduction}
A model may stop improving because its architecture is inadequate, or because its input protocol does not contain the information required by the label. The distinction is especially important when labels summarize a long horizon while inputs are snapshots. A clinical score may refer to symptoms over weeks, a maintenance label to an operating cycle, and an ecological target to a season, while the model receives one interview, one inspection, or a few images. In such settings, a benchmark ceiling can be a sampling-protocol ceiling rather than a model-capacity ceiling.

The protocol has at least three different dimensions. The number of segments $M$ controls how precisely a fixed temporal support is measured; a window length $w$ controls the support of one recording; and the number and locations of windows $D$ control temporal coverage. These quantities are not interchangeable. More segments can denoise a snapshot and strengthen inference about a stable individual trait, but they do not reveal state innovations outside the observed support. Conversely, high cross-sectional accuracy can be driven by stable between-person differences even when the model has little sensitivity to within-person temporal change.

Our nonlinear example is an occupation-time label: the fraction of a horizon during which a latent state exceeds a threshold. This abstraction captures frequency-type targets such as the proportion of time spent in a symptomatic, unsafe, or anomalous state. Unlike a temporal mean, an occupation label depends on the complete correlation structure after thresholding. It therefore exposes a general point: the statistical timescale of an input is defined jointly by the latent dynamics and the label functional.

We make three contributions. First, we prove a trait--state asymptotic decomposition for Gaussian labels formed by temporal aggregation. It separates an $O(1)$ cross-sectional channel from an $O(T^{-1})$ state channel and yields a nonzero snapshot-prediction limit. Second, we derive task-dependent effective temporal spans for mean and occupation labels; the latter depends on all higher-order correlation times, so matching the usual integral correlation time does not match temporal information. Third, we prove \emph{boundary localization}: for occupation labels, state-driven label variance is largest for individuals whose stable trait lies at the threshold, while local-window efficiency lacks the same multiplicative concentration. Experiments verify these results and directly compare same-time segmentation with dispersed temporal coverage under an equal segment budget.

\section{Related Work}
Generalizability theory decomposes object, occasion, rater, and measurement facets and uses D-studies to compare acquisition designs \citep{cronbach1972dependability,shavelson1991primer,brennan2001generalizability}. Latent state--trait theory likewise separates enduring traits from occasion-specific states \citep{steyer1999latent}. Variance-component and random-effects methods provide the corresponding estimation machinery \citep{searle1992variance,robinson1991blup}. Our model uses the same conceptual separation but makes the prediction target a nonlinear functional of a correlated continuous-time state.

Longitudinal and functional-data methods study covariance estimation, trajectory recovery, and measurement placement \citep{diggle2002longitudinal,ramsay2005functional,ji2017optimal}. Classical measurement-error theory distinguishes biological variation from noisy observation \citep{fuller1987measurement,carroll2006measurement}. Sampling design and pseudoreplication also warn that correlated repeats do not equal independent support \citep{kish1965survey,cochran1977sampling,hurlbert1984pseudoreplication,pyper1998autocorrelation}, while Bayesian design formalizes the value of choosing informative observations \citep{chaloner1995bayesian}. We instead derive the maximum predictive information of a fixed protocol and show that its temporal value depends on the label functional.

Weak supervision, learning from label proportions, and multiple-instance learning attach labels to bags rather than local instances \citep{dietterich1997multiple,quadrianto2009labels,zhou2018brief,ilse2018attention}. These frameworks do not by themselves determine what fraction of a long-horizon label is observable through a short temporal support. Our protocol ceiling is complementary: it bounds every learner using the same observations, independently of architecture or training-set size.

Occupation functionals and Gaussian level crossings have a substantial probability literature \citep{kratz2006level,adler2007random}; Hermite expansions and limit theory expose the role of higher-order correlations \citep{breuer1983central}. We use Plackett's Gaussian derivative identity for threshold covariances \citep{plackett1954reduction}. Discrete approximation of occupation functionals has sharp error theory \citep{altmeyer2018estimation,altmeyer2021approximation}, and Gaussian-process excursion-volume uncertainty and sequential design are well studied \citep{vazquez2006excursion,bect2012sequential,azzimonti2016quantifying,bect2019supermartingale}. We do not claim novelty for excursion-volume posterior inference or adaptive point placement. Our focus is multiple independent objects, a stable trait plus correlated state, sparse noisy windows, and supervised labels whose temporal meaning changes with the aggregation functional.

\section{Temporal-Aggregate Learning}
For object $i$, let
\begin{equation}
Z_i(t)=\sqrt{\alpha}\,M_i+\sqrt{1-\alpha}\,X_i(t),\qquad t\in[0,T],
\label{eq:model}
\end{equation}
where $M_i\sim\cN(0,1)$ is a stable trait, $X_i$ is a zero-mean unit-variance stationary Gaussian process, and the two are independent. Write
\begin{align}
\rho(u)&=\Cov\{X_i(t),X_i(t+u)\},\\
r_\alpha(u)&=\alpha+(1-\alpha)\rho(u).
\end{align}
Thus $Z_i(t)$ is marginally standard normal and $\alpha$ is its long-lag correlation. The unnormalized model $\mu_i+\sigma_sX_i(t)$ is equivalent after rescaling, with $\alpha=\sigma_\mu^2/(\sigma_\mu^2+\sigma_s^2)$.

For $g\in L^2(\phi)$, define the temporal-aggregate label
\begin{equation}
\Theta_{i,g,T}=\frac{1}{T}\int_0^T g\{Z_i(t)\}\dd t.
\label{eq:label}
\end{equation}
We use $g(z)=z$ for a mean label and $g_c(z)=\Ind\{z>c\}$ for an occupation-time label. A protocol $\pi$ contains noisy linear window observations $Y_{id}=L_d Z_i+\varepsilon_{id}$. A window of length $w_d$ centered at $t_d$ has $L_dZ=w_d^{-1}\int_{t_d-w_d/2}^{t_d+w_d/2}Z(u)\dd u$. We distinguish the number of time windows $D$, their support lengths $w_d$, and the number of segments $M$ used to estimate a fixed window. Increasing $M$ can reduce measurement noise, but does not change the observed temporal set. Their primary roles are
\[
\begin{aligned}
M&:\text{ precision}, & w&:\text{ local support},\\
(D,t_{1:D})&:\text{ temporal coverage}.&&
\end{aligned}
\]
Only $w$ and $(D,t_{1:D})$ change temporal support; $M$ refines measurement on support already observed.

Under squared loss, the optimal protocol risk and explainability are
\begin{equation}
R_\pi^*=\E\Var(\Theta_{g,T}\mid Y_\pi),\qquad
\cI_\pi=1-\frac{R_\pi^*}{\Var(\Theta_{g,T})}.
\label{eq:risk}
\end{equation}
The quantity $\cI_\pi$ is a protocol-level theoretical $R^2$, not the performance of a particular architecture. We also condition on $M_i$ to isolate the state channel, preventing high cross-sectional prediction from being mistaken for successful tracking of within-person dynamics.

\section{Protocol Risk as a Common Tool}
Let $(U,V_r)$ be standard bivariate normal with correlation $r$, and define $C_g(r)=\Cov\{g(U),g(V_r)\}$. For linear Gaussian observations, let $k_\pi(t)=\Cov\{Z(t),Y_\pi\}$, $V_\pi=\Var(Y_\pi)$, and
\begin{equation}
m_\pi(t)=k_\pi(t)^\top V_\pi^{-1}Y_\pi,\qquad
q_\pi(s,t)=k_\pi(s)^\top V_\pi^{-1}k_\pi(t).
\label{eq:qpi}
\end{equation}
Here $q_\pi$ is the covariance explained by the protocol; the posterior covariance is $r_\alpha(s-t)-q_\pi(s,t)$.

\begin{proposition}[Exact protocol-conditioned risk]
For any $g\in L^2(\phi)$,
\begin{equation}
R_\pi^*=\frac{1}{T^2}\int_0^T\!\int_0^T
\left[C_g\{r_\alpha(s-t)\}-C_g\{q_\pi(s,t)\}\right]\dd s\dd t.
\label{eq:riskidentity}
\end{equation}
The Bayes predictor is $T^{-1}\int_0^T\E[g\{Z(t)\}\mid Y_\pi]\dd t$.
\end{proposition}
\begin{proof}
Fubini's theorem gives
\[
\Var(\Theta_{g,T})=\frac{1}{T^2}\int_0^T\!\int_0^T
C_g\{r_\alpha(s-t)\}\dd s\dd t.
\]
Now draw posterior-process replicas $Z^{(1)}$ and $Z^{(2)}$ independently conditional on $Y_\pi$. Their unconditional marginals equal that of $Z$, while the Gaussian conditioning formula gives
\begin{align}
\Cov\{Z^{(1)}(s),Z^{(2)}(t)\}
&=\Cov\big\{\E(Z(s)\mid Y_\pi),\notag\\
&\hspace{3.4em}\E(Z(t)\mid Y_\pi)\big\}\notag\\
&=q_\pi(s,t).
\end{align}
Conditional independence implies
\begin{align}
\E[g\{Z^{(1)}(s)\}g\{Z^{(2)}(t)\}]
&=\E\big[\E\{g(Z(s))\mid Y_\pi\}\notag\\
&\hspace{2.7em}\cdot\E\{g(Z(t))\mid Y_\pi\}\big],
\end{align}
so the covariance of the two posterior means is $C_g\{q_\pi(s,t)\}$. Integrating yields
\[
\Var\{\E(\Theta_{g,T}\mid Y_\pi)\}
=T^{-2}\!\int_0^T\!\!\int_0^T C_g\{q_\pi(s,t)\}\dd s\dd t.
\]
The law of total variance subtracts this quantity from $\Var(\Theta_{g,T})$, proving Eq.~\eqref{eq:riskidentity}; conditional Fubini gives the stated predictor.
\end{proof}
For occupation labels, $C_g(r)=G_c(r)=\Pr(U>c,V_r>c)-\PhiBar(c)^2$; at $c=0$, $G_0(r)=\arcsin(r)/(2\pi)$.

\begin{remark}[Segments are useful, but they are not time]
Let $\cF_M$ be the information generated by an increasingly fine segmentation of a fixed observed set $W$. If $\cF_M\uparrow\sigma\{Z(t):t\in W\}$, L\'evy's upward theorem gives
\[
R_M^*\downarrow\E\Var\!\left(\Theta_{g,T}\mid\sigma\{Z(t):t\in W\}\right).
\]
More segments can reduce sensor noise and improve trait estimation, so the plateau may be lower than the risk of a coarse recording. They cannot reveal state innovations outside $W$; the title's distinction therefore concerns temporal state coverage, not the usefulness of repeated measurements for precision.
\end{remark}

\section{Protocol Ceilings and Learning Gaps}
For any measurable predictor $f(Y_\pi)$, the orthogonality of conditional expectation gives the exact decomposition
\begin{equation}
\E\{\Theta_{g,T}-f(Y_\pi)\}^2
=R_\pi^*+\E\{f(Y_\pi)-f_\pi^*(Y_\pi)\}^2.
\label{eq:modelgap}
\end{equation}
Equivalently, its population coefficient of determination satisfies
\begin{equation}
R_\pi^2(f)=\cI_\pi-
\frac{\E\{f(Y_\pi)-f_\pi^*(Y_\pi)\}^2}{\Var(\Theta_{g,T})}.
\label{eq:r2gap}
\end{equation}
The first term is fixed by acquisition; only the second can be reduced by architecture, optimization, or more training objects. This separates two empirically similar forms of saturation: a large model gap under an informative protocol, and a small model gap near a low protocol ceiling. When $\cI_\pi$ can be estimated from a calibrated temporal model, $R_\pi^2(f)/\cI_\pi$ is a descriptive ceiling-utilization ratio; it is meaningful only for the same target, loss, and protocol assumptions.

Cross-sectional and monitoring benchmarks also answer different questions. Write
\[
\Theta_{g,T}=h_g(M)+\Delta_{g,T},\qquad \E(\Delta_{g,T}\mid M)=0.
\]
Total explainability includes prediction of the stable $h_g(M)$ channel. A trait-conditioned state ceiling instead evaluates how much of $\Delta_{g,T}$ is recoverable from temporal windows. High total $R^2$ can therefore coexist with weak within-person tracking; reporting only the former can make protocol-limited state learning look like successful dynamics modeling.

\section{Trait--State Limits}
Center $g$ and expand it in probabilists' Hermite polynomials:
\begin{align}
g(z)-\E g(U)&=\sum_{k\ge1}\frac{a_k(g)}{k!}H_k(z),\\
C_g(r)&=\sum_{k\ge1}\frac{a_k(g)^2}{k!}r^k.
\label{eq:hermite}
\end{align}
Let $\tau_j=\int_0^\infty\rho(u)^j\dd u$.

\begin{theorem}[Trait--state decomposition]
Assume $0\le\rho\le1$ and the series below is finite. Then
\begin{equation}
\Var(\Theta_{g,T})=V_{\rm trait}(g)+\frac{A_{\rm state}(g)}{T}+o(T^{-1}),
\label{eq:decomp}
\end{equation}
where
\begin{align}
V_{\rm trait}(g)&=C_g(\alpha)
=\Var\!\left[\E\{g(Z(t))\mid M\}\right],\\
A_{\rm state}(g)&=2\int_0^\infty
\left[C_g\{\alpha+(1-\alpha)\rho(u)\}-C_g(\alpha)\right]\dd u\\
&=2\sum_{k\ge1}\frac{a_k(g)^2}{k!}
\sum_{j=1}^k \binom{k}{j}\alpha^{k-j}(1-\alpha)^j\tau_j.
\label{eq:astate}
\end{align}
\end{theorem}
\begin{proof}
Stationarity and symmetry give
\[
\Var(\Theta_{g,T})=\frac{2}{T^2}\int_0^T(T-u)C_g\{r_\alpha(u)\}\dd u.
\]
Writing $\Delta_g(u)=C_g\{r_\alpha(u)\}-C_g(\alpha)$ separates this as
\[
C_g(\alpha)+\frac{2}{T}\int_0^\infty
\Ind\{u\le T\}(1-u/T)\Delta_g(u)\dd u.
\]
The assumed summability makes $\Delta_g$ integrable, so dominated convergence gives $A_{\rm state}(g)/T+o(T^{-1})$. Substituting Eq.~\eqref{eq:hermite}, expanding
$[\alpha+(1-\alpha)\rho(u)]^k-\alpha^k$, and applying Tonelli's theorem gives Eq.~\eqref{eq:astate}. Finally,
\[
\E\!\left[H_k\{\sqrt\alpha M+\sqrt{1-\alpha}X\}\mid M\right]
=\alpha^{k/2}H_k(M),
\]
and Hermite orthogonality yields
$\Var[\E\{g(Z(t))\mid M\}]=\sum_{k\ge1}a_k(g)^2\alpha^k/k!=C_g(\alpha)$.
\end{proof}
The first term is a stable $O(1)$ cross-sectional channel; the second is finite-horizon state variation. For one noisy point observation $Y=Z(T/2)+\varepsilon$, $\varepsilon\sim\cN(0,\nu^2)$,
\begin{equation}
\lim_{T\to\infty}\cI_Y=
\frac{C_g\{\alpha^2/(1+\nu^2)\}}{C_g(\alpha)}>0,
\qquad \alpha>0.
\label{eq:snapshotlimit}
\end{equation}
Indeed, $q_Y(s,t)=r_\alpha(|s-T/2|)r_\alpha(|t-T/2|)/(1+\nu^2)$ approaches $\alpha^2/(1+\nu^2)$ away from a vanishing boundary fraction. Applying the exact risk identity and Ces\`aro convergence proves Eq.~\eqref{eq:snapshotlimit}.
Without a trait, a fixed snapshot explains a vanishing fraction of an increasingly long label. With a trait, cross-sectional prediction retains an $O(1)$ channel, so apparent benchmark performance can remain substantial without learning temporal state dynamics.

\begin{corollary}[What $D$ and $M$ buy in the trait channel]
For a mean label as $T\to\infty$, suppose $D$ occasions are separated enough that their state terms are independent, and each occasion averages $M$ segments with raw segment-noise variance $\sigma_\varepsilon^2$. Then the trait explainability is
\begin{equation}
\cI_{\rm trait}(D,M)=
\frac{\alpha}{\alpha+(1-\alpha)/D+\sigma_\varepsilon^2/(DM)}.
\label{eq:traitdm}
\end{equation}
\end{corollary}
\begin{proof}
The protocol average is
\[
\bar Y=\sqrt\alpha M_i+\sqrt{1-\alpha}\,\bar X_D+\bar\varepsilon.
\]
The independent terms have variances $\alpha$, $(1-\alpha)/D$, and
$\sigma_\varepsilon^2/(DM)$. The Gaussian regression $R^2$ for predicting
$\sqrt\alpha M_i$ from $\bar Y$ is
$\Cov(\sqrt\alpha M_i,\bar Y)^2/
[\Var(\sqrt\alpha M_i)\Var(\bar Y)]$, giving Eq.~\eqref{eq:traitdm}.
\end{proof}
Same-time replication corresponds to $D=1$: it removes measurement noise as $M$ grows but plateaus at $\alpha$. Temporally separated occasions also average transient state noise and can approach unit trait explainability. Thus more segments are useful for precision, while more time supplies an additional source of information.

Importantly, Eq.~\eqref{eq:traitdm} does \emph{not} require the short-lag state kernel. After standardization, a conventional two-occasion test--retest design at a lag where the transient state correlation is negligible identifies $\alpha$ from cross-occasion covariance and $\sigma_\varepsilon^2$ from observed variance (or from within-occasion segments). Hence any suitable repeated-measurement dataset can already produce the trait-channel ceiling. Estimating $\rho$ and $\tau_{k\ge2}$ is needed only for the state-channel quantities below.

For mean labels, $V_{\rm trait}=\alpha$ and $A_{\rm state}=2(1-\alpha)\tau_1$. For occupation labels, every Hermite order contributes, making the state timescale task dependent. For an OU state kernel $\rho(u)=e^{-u/\tau}$, $\tau_j=\tau/j$. Conditional on a trait value, an occupation label has standardized state threshold $a$ and
\begin{equation}
A_a=2\tau\int_0^1\frac{G_a(r)}{r}\dd r,
\qquad
A_0=\frac{\tau\log 2}{2}.
\label{eq:ouclosed}
\end{equation}
The closed form at the boundary independently matches the Hermite series.

\section{Task-Dependent Effective Time}
To isolate temporal information beyond the stable trait, condition on $M=m$. Define
\[
h_g(m)=\E_X\!\left[g\{\sqrt\alpha m+\sqrt{1-\alpha}X(t)\}\right]
\]
and subtract this trait-conditional mean. For an occupation label, set $a=(c-\sqrt\alpha m)/\sqrt{1-\alpha}$. Let $W_w$ be a standardized noisy state average over a window of length $w$, let $r_w(t)=\Corr\{X(t),W_w\}$, and define $J_k(w)=\int_{\R}r_w(t)^k\dd t$.

\begin{theorem}[Task-dependent state-effective span]
Suppose the window remains interior as $T\to\infty$, $r_w^k\in L^1(\R)$
for every Hermite order with nonzero weight, and the weighted $\tau_k$ and
$J_k(w)^2$ series below are finite. Then
\begin{equation}
\cI_{\rm state}(w\mid m)=\frac{\ell_g(w;m)}{T}+o(T^{-1}).
\label{eq:effective}
\end{equation}
For a mean label,
\begin{align}
\ell_{\rm mean}(w)&=\frac{2\tau_1}{\eta(w)+\nu^2},\\
\eta(w)&=\frac{2}{w^2}\int_0^w(w-u)\rho(u)\dd u.
\label{eq:ellmean}
\end{align}
For an occupation label,
\begin{equation}
\ell_a(w)=
\frac{\displaystyle\sum_{k\ge1}\frac{H_{k-1}(a)^2}{k!}J_k(w)^2}
{\displaystyle2\sum_{k\ge1}\frac{H_{k-1}(a)^2}{k!}\tau_k}.
\label{eq:ellocc}
\end{equation}
\end{theorem}
\begin{proof}
Conditional on $M=m$, set
$g_m(x)=g(\sqrt\alpha m+\sqrt{1-\alpha}x)$ and expand
\[
g_m\{X(t)\}-\E g_m(U)
=\sum_{k\ge1}\frac{\beta_k(m)}{k!}H_k\{X(t)\}.
\]
Hermite orthogonality gives
\begin{align}
\Var(\Theta_{g,T}\mid M=m)
&=\frac{A_g(m)}{T}+o(T^{-1}),\\
A_g(m)&=2\sum_{k\ge1}\frac{\beta_k(m)^2}{k!}\tau_k.
\end{align}
Because $W_w$ is standardized, Gaussian regression gives
\[
\E\{H_k(X(t))\mid W_w\}=r_w(t)^kH_k(W_w).
\]
Therefore
\begin{align}
\Var\{\E(\Theta_{g,T}\mid W_w,M=m)\}
&=\frac{1}{T^2}\sum_{k\ge1}\frac{\beta_k(m)^2}{k!}\notag\\
&\quad{}\times
\left\{\int_0^T r_w(t)^k\dd t\right\}^{\!2}.
\end{align}
For an interior window with integrable correlation-profile tails, the truncated integral converges to $J_k(w)$; termwise convergence follows from the stated summability. Dividing explained variance by total state variance proves Eq.~\eqref{eq:effective}. For an occupation label,
$\beta_k(m)=\phi(a)H_{k-1}(a)$, which gives Eq.~\eqref{eq:ellocc}. For a mean label only $k=1$ remains. Since
$J_1(w)=2\tau_1/\sqrt{\eta(w)+\nu^2}$, the ratio
$J_1(w)^2/(2\tau_1)$ gives Eq.~\eqref{eq:ellmean}.
\end{proof}

For mean labels, only the first Hermite order remains. Occupation labels use all $\tau_k$ and all window profiles $J_k$; the common $\phi(a)^2$ factor cancels from Eq.~\eqref{eq:ellocc}. Thus physical duration $w$ is not the statistical value of a window, and two kernels with equal $\tau_1$ can still define different occupation-time information. For a noise-free OU window, $\ell_{\rm mean}(w)=w^2/[w-\tau(1-e^{-w/\tau})]$, approaching $2\tau$ for $w\ll\tau$ and $w$ for $w\gg\tau$.

\begin{corollary}[Equal segment budget]
Write $\ell_g(w;\nu^2,m)$ when the averaged window has noise variance $\nu^2$. Allocate $N$ independent raw segments either to one fixed window or to $N$ mutually separated windows. In the sparse long-horizon regime,
\begin{align}
\cI_{\rm same}(N\mid m)
&=\frac{\ell_g(w;\sigma_\varepsilon^2/N,m)}{T}+o(T^{-1})\notag\\
&\longrightarrow\frac{\ell_g(w;0,m)}{T},\label{eq:samebudget}\\
\cI_{\rm dispersed}(N\mid m)
&=\frac{N\ell_g(w;\sigma_\varepsilon^2,m)}{T}+o(N/T),\label{eq:dispersedbudget}
\end{align}
until the sparse additivity approximation approaches saturation.
\end{corollary}
\begin{proof}
For same-time segments, averaging $N$ independent sensor errors changes only the local noise variance from $\sigma_\varepsilon^2$ to $\sigma_\varepsilon^2/N$; Theorem~2 gives Eq.~\eqref{eq:samebudget}. For separated windows, cross-window covariance terms are negligible in the sparse regime, so their explained state variances add to first order. Summing $N$ equal contributions and dividing by $A_g(m)/T$ gives Eq.~\eqref{eq:dispersedbudget}.
\end{proof}
Thus repeated segments can exhaust local measurement noise but have a finite state-information limit; separated windows purchase additional state support. For a mean label and point-like windows, the two expressions reduce to
\[
\frac{2\tau_1}{T(1+\sigma_\varepsilon^2/N)}
\quad\text{and}\quad
\frac{2N\tau_1}{T(1+\sigma_\varepsilon^2)},
\]
respectively.

\begin{theorem}[Boundary localization]
Assume $\rho(u)\ge0$. For the trait-conditioned threshold $a$,
\begin{equation}
G_a(r)=\int_0^r\frac{\exp\{-a^2/(1+s)\}}{2\pi\sqrt{1-s^2}}\dd s.
\label{eq:plackett}
\end{equation}
Hence $A_a=2\int_0^\infty G_a\{\rho(u)\}\dd u$ is even and strictly decreases with $|a|$ whenever $\rho$ is positive on a set of nonzero measure; it is maximized at $a=0$. The factor $\phi(a)^2$ cancels from Eq.~\eqref{eq:ellocc}, so $\ell_a(w)$ varies only through reweighting of Hermite orders.
\end{theorem}
\begin{proof}
Plackett's identity gives
$\partial_r\Pr(U>a,V_r>a)=
\exp\{-a^2/(1+r)\}/[2\pi\sqrt{1-r^2}]$.
At $r=0$ the excess covariance is zero, so integration gives Eq.~\eqref{eq:plackett}. For every $r>0$ its integrand is even in $a$ and strictly decreases with $|a|$. Integration over any nonnegative $\rho$ preserves these properties and is strict when $\rho>0$ on a set of positive measure. Finally, the occupation Hermite coefficient is $\phi(a)H_{k-1}(a)$; the common squared factor $\phi(a)^2$ cancels between the numerator and denominator of Eq.~\eqref{eq:ellocc}.
\end{proof}
Threshold-near individuals therefore have more state-driven label variance to explain. In the OU protocols studied below, $\ell_a(w)$ decays far more slowly than $A_a$. For $D$ separated sparse windows, the conditional residual state risk is approximately
\begin{equation}
R_{\rm state}(D,w\mid m)\approx
\frac{A_a}{T}\left\{1-\frac{D\ell_a(w)}{T}\right\}.
\label{eq:boundaryresidual}
\end{equation}
Consequently, absolute residual error is largest near the threshold even when the fraction of state variance explained changes much less. The distinction is important: boundary-near individuals are not necessarily observed with a much less efficient window; rather, their labels contain substantially more state variation that the protocol must explain.

\section{Implications for ML Benchmarks}
Equations~\eqref{eq:modelgap}--\eqref{eq:r2gap} give a ceiling-aware interpretation of benchmark progress. First, a reported score should be compared with the information available under the benchmark's own $D$, $w$, $M$, and temporal placement, rather than with the unattainable value $R^2=1$. Second, total cross-sectional performance and state-tracking performance should be reported separately whenever the scientific claim concerns change. Subject-wise centering, repeated labels, or a calibrated latent trait can define the state target $\Delta_{g,T}$; without such a decomposition, a model may rank individuals well while failing to monitor them.

Third, acquisition and architecture should be treated as distinct experimental axes. Increasing training-set size estimates $f_\pi^*$ more accurately but does not alter $\cI_\pi$; changing the observation protocol alters the ceiling itself. A useful ablation therefore holds the raw measurement budget fixed while reallocating it between same-time replication and temporal coverage, as in Eqs.~\eqref{eq:samebudget}--\eqref{eq:dispersedbudget}. The calibration burden is asymmetric. The trait-channel ceiling in Eq.~\eqref{eq:traitdm} uses only $\alpha$ and $\sigma_\varepsilon^2$, available from ordinary test--retest data under the model, and can therefore be reported now for many existing benchmarks. Only state-channel claims require short-lag estimation of $\rho$ and its higher-order integrals; for that purpose, a small densely sampled subset can be more informative than another large cross-sectional sample collected under the same snapshot protocol.

\section{Experiments and Benchmark Implications}
We test analytic predictions rather than compare architectures. OU paths use exact transitions, labels are computed on a fine grid, and predictors are exact posterior expectations under the discretized process. The main table uses 50 independent repetitions of 2,000 objects per scenario; the equal-budget experiment uses 30 repetitions of 1,500 objects. Monte Carlo means are reported with 95\% half-widths.

\begin{figure*}[t]
\centering
\begin{minipage}[t]{0.49\textwidth}
\centering
\includegraphics[width=\linewidth]{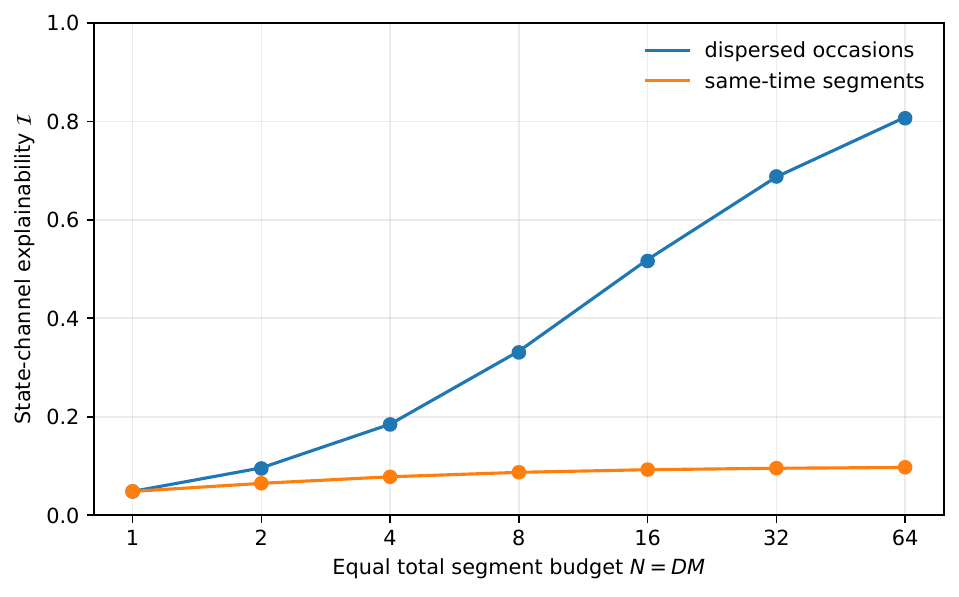}
\par\smallskip{\small (a) Equal raw-segment budget: precision versus coverage.\par}
\end{minipage}
\hfill
\begin{minipage}[t]{0.49\textwidth}
\centering
\includegraphics[width=\linewidth]{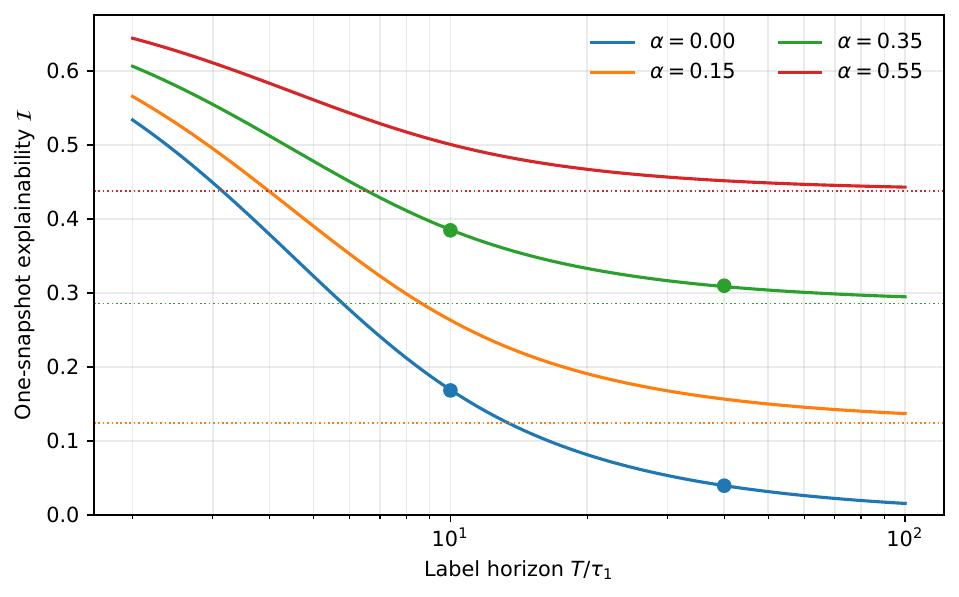}
\par\smallskip{\small (b) Snapshot ceilings over label horizons and trait shares.\par}
\end{minipage}
\caption{Protocol ceilings, not model training curves. In (a), exact state-channel ceilings and Monte Carlo estimates show that same-time replication rapidly saturates while dispersed occasions continue to add temporal information. In (b), stable traits create nonzero long-horizon snapshot ceilings; dotted lines are the $T\to\infty$ limits.}
\label{fig:ceilings}
\end{figure*}

\paragraph{Equal segment budget: $D$ versus $M$.}
Figure~\ref{fig:ceilings}a isolates the state channel ($\alpha=0$) with an occupation label, $T/\tau=20$, and unit noise per raw segment. For each $N=DM$, the same-time protocol uses $D=1,M=N$ at the midpoint, whereas the coverage protocol uses $D=N,M=1$ at evenly spaced times. For $c=0$, the exact ceiling is computed from Eq.~\eqref{eq:riskidentity} with $C_g(r)=\arcsin(r)/(2\pi)$. Same-time replication reaches only $0.097$ at $N=64$; dispersed occasions reach $0.808$ using the same 64 raw segments.

\paragraph{Benchmark ceilings from the trait channel.}
Figure~\ref{fig:ceilings}b verifies Eq.~\eqref{eq:snapshotlimit}: when $\alpha=0$, one-snapshot explainability vanishes as the label horizon grows, whereas nonzero trait shares converge to positive plateaus. For a zero-threshold occupation label with $T/\tau=14$ and point-noise variance $0.2$, Eq.~\eqref{eq:riskidentity} gives ceilings $\cI=0.119$, $0.256$, and $0.355$ for $\alpha=0$, $0.20$, and $0.35$. A benchmark stalled near $R^2=0.25$ may therefore be close to its acquisition ceiling under a moderate trait channel. These values are model-based illustrations, not estimates for a particular dataset.

\paragraph{Task dependence and boundary localization.}
Figure~\ref{fig:task}a plots $\ell_{\rm occ}/\ell_{\rm mean}$, making visible that OU and Mat\'ern-$3/2$ kernels matched to the same $\tau_1$ assign different relative value to the same window. Figure~\ref{fig:task}b separates magnitude from efficiency: $A_a$ is sharply concentrated near $a=0$, whereas $\ell_a(w)$ decays much more slowly because the universal $\phi(a)^2$ factor cancels.

\begin{figure}[t]
\centering
\includegraphics[width=\linewidth]{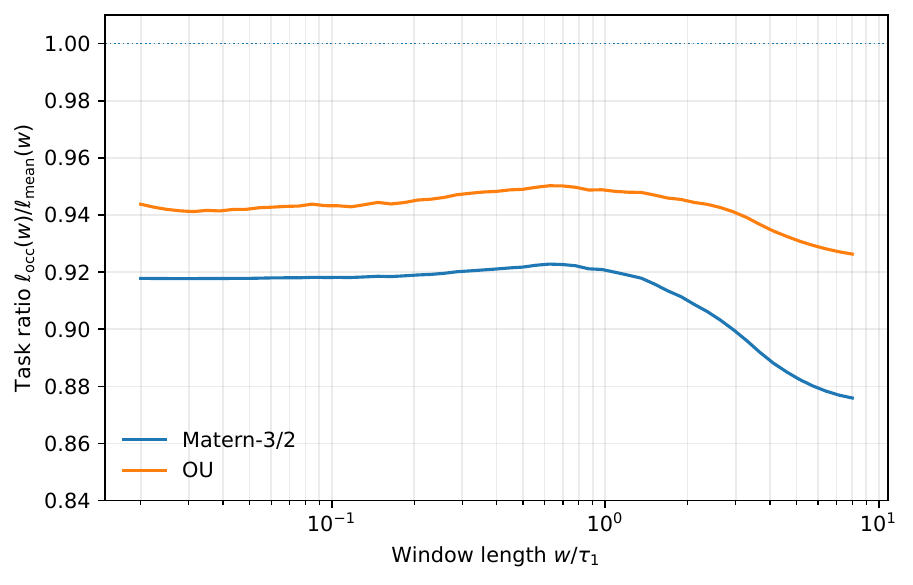}
\par\smallskip{\small (a) Occupation-to-mean effective-span ratio.\par}
\smallskip
\includegraphics[width=\linewidth]{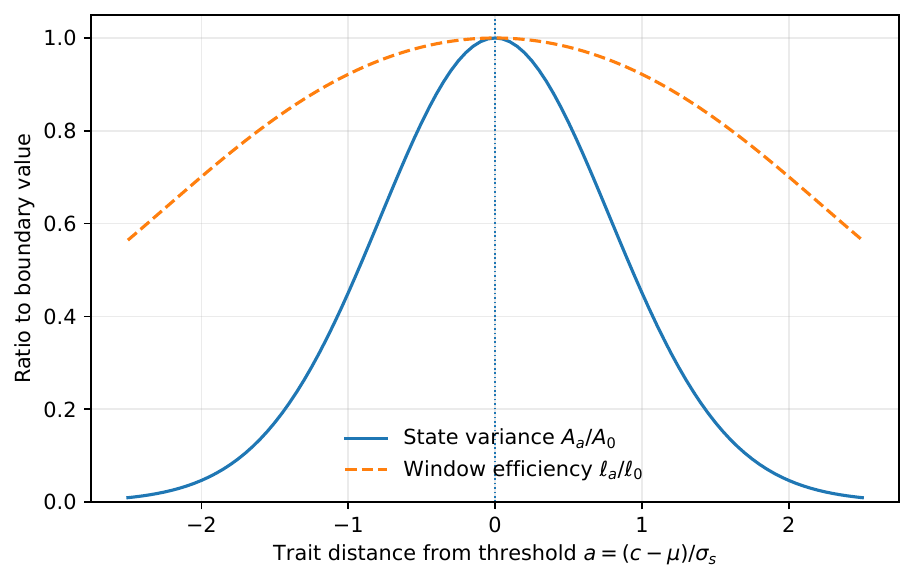}
\par\smallskip{\small (b) State variance localizes more strongly than efficiency.\par}
\caption{The label defines temporal information. Both kernels in panel (a) have $\tau_1=1$; panel (b) uses an OU window with $w/\tau_1=0.5$.}
\label{fig:task}
\end{figure}

\begin{table}[t]
\centering
\scriptsize
\resizebox{\columnwidth}{!}{%
\begin{tabular}{@{}rrrrrrr@{}}
\toprule
$\alpha$ & $c$ & $T/\tau$ & \multicolumn{2}{c}{$\Var(\Theta)$} & \multicolumn{2}{c}{$\cI$}\\
&&& Theory & MC $\pm$ 95\% half-width & Theory & MC $\pm$ 95\% half-width\\
\midrule
0.00 & 0 & 10 & 0.0314 & 0.0314 $\pm$ 0.0003 & 0.169 & 0.168 $\pm$ 0.0016 \\
0.00 & 0 & 40 & 0.0085 & 0.0085 $\pm$ 0.0001 & 0.040 & 0.040 $\pm$ 0.0003 \\
0.35 & 0 & 10 & 0.0800 & 0.0800 $\pm$ 0.0005 & 0.385 & 0.385 $\pm$ 0.0024 \\
0.35 & 0 & 40 & 0.0631 & 0.0629 $\pm$ 0.0003 & 0.309 & 0.310 $\pm$ 0.0026 \\
0.00 & 1 & 10 & 0.0143 & 0.0143 $\pm$ 0.0001 & 0.150 & 0.149 $\pm$ 0.0022 \\
0.00 & 1 & 40 & 0.0038 & 0.0038 $\pm$ 0.0000 & 0.035 & 0.036 $\pm$ 0.0005 \\
0.35 & 1 & 10 & 0.0363 & 0.0364 $\pm$ 0.0005 & 0.347 & 0.348 $\pm$ 0.0044 \\
0.35 & 1 & 40 & 0.0275 & 0.0276 $\pm$ 0.0004 & 0.278 & 0.275 $\pm$ 0.0040 \\
\bottomrule

\end{tabular}
}
\caption{Exact formulas versus Monte Carlo means. Each cell uses 50 repetitions of 2,000 objects; intervals are 95\% Monte Carlo half-widths.}
\label{tab:mc}
\end{table}

\section{Limitations and Conclusion}
The two channels have different calibration requirements. The trait-channel ceiling does not require a densely sampled calibration subset: under the normalized model, any suitable repeated-measurement or two-occasion test--retest dataset can estimate $\alpha$ and $\sigma_\varepsilon^2$ and directly evaluate Eq.~\eqref{eq:traitdm}. Purely cross-sectional data cannot identify that decomposition, but repeated measurements suffice without resolving the short-lag kernel.

The state channel is more demanding. Mean-label state variation depends on $\tau_1$, whereas occupation labels use $\tau_{k\ge2}$ and the full window profile. Sparse windows separated far beyond the correlation scale cannot identify this short-lag structure, so state-channel analysis requires a densely sampled subset, external short-lag longitudinal data, or a justified parametric kernel family. Figure~\ref{fig:task}a makes the issue visible: equal $\tau_1$ does not imply equal occupation-time information.

The occupation model is an abstraction of frequency-type labels, not a claim that every observed score is generated by one thresholded Gaussian state. The state-effective span is conditional on the trait; an average span defined as a ratio of expected explained and total state variance is weighted toward individuals with larger state variance and is not the span of a typical individual.

The central conclusion is not that segments are useless. Segments improve denoising and can strengthen the trait channel. They are not additional time. Long-horizon benchmark performance combines an $O(1)$ trait channel with a task-dependent state channel controlled by the aggregation functional, correlation structure, and temporal support. Consequently, a performance ceiling that appears architectural may instead be imposed by acquisition, and high cross-sectional accuracy need not imply that a model has learned temporal state dynamics.

\bibliography{references}

\clearpage
\onecolumn
\pagestyle{plain}
\raggedbottom
\setcounter{secnumdepth}{2}
\setcounter{section}{0}
\setcounter{subsection}{0}
\setcounter{equation}{0}
\setcounter{theorem}{0}
\setcounter{proposition}{0}
\setcounter{corollary}{0}
\setcounter{lemma}{0}
\setcounter{assumption}{0}
\renewcommand{\thesection}{S\arabic{section}}
\renewcommand{\thesubsection}{\thesection.\arabic{subsection}}
\renewcommand{\theequation}{S\arabic{equation}}
\renewcommand{\thetheorem}{S\arabic{theorem}}
\renewcommand{\theproposition}{S\arabic{proposition}}
\renewcommand{\thecorollary}{S\arabic{corollary}}
\renewcommand{\thelemma}{S\arabic{lemma}}
\renewcommand{\theassumption}{S\arabic{assumption}}

\begin{center}
{\LARGE\bfseries Technical Supplement\par}
\vspace{0.35em}
{\Large\itshape The Label Defines the Timescale: Trait--State Limits of
Temporal-Aggregate Learning\par}
\vspace{0.65em}
{\large Xizhe Zhang\par}
\vspace{0.2em}
ORCID: 0000-0002-8684-4591 \quad \textbullet \quad
zhangxizhe@gmail.com
\end{center}
\vspace{0.8em}

\begingroup
\setlength{\parskip}{0.15em}
\section{Model, Notation, and Regularity Conditions}
For each independent object, the standardized latent process is
\begin{equation}
Z(t)=\sqrt\alpha M+\sqrt{1-\alpha}X(t),\qquad t\in[0,T],
\label{eq:s-model}
\end{equation}
where $M\sim\cN(0,1)$, $X$ is a zero-mean unit-variance stationary Gaussian process, and $M\perp X$.  The state correlation is $\rho(u)=\Cov\{X(t),X(t+u)\}$, and the total-process correlation is
\begin{equation}
r_\alpha(u)=\alpha+(1-\alpha)\rho(u).
\end{equation}
For $g\in L^2(\phi)$, where $\phi$ is the standard-normal density, define
\begin{equation}
\Theta_{g,T}=\frac1T\int_0^T g\{Z(t)\}\dd t.
\end{equation}
The observation protocol is a finite-dimensional linear Gaussian measurement
\begin{equation}
Y_\pi=L_\pi Z+\varepsilon,\qquad \varepsilon\sim\cN(0,R_\pi),
\label{eq:s-obs}
\end{equation}
independent of $Z$.  Window averages are a special case.

We use the following sufficient conditions.  They are stronger than necessary but make every interchange explicit.

\begin{assumption}[Summability for Main Theorem 1]
\label{ass:s-summability}
The function $g$ is square integrable under the standard-normal law.  Whenever a long-horizon expansion is invoked,
\begin{equation}
\Delta_g(u)=C_g\{\alpha+(1-\alpha)\rho(u)\}-C_g(\alpha)
\end{equation}
is absolutely integrable on $[0,\infty)$.  For an $O(T^{-2})$ remainder, we additionally assume $\int_0^\infty u|\Delta_g(u)|\dd u<\infty$.
\end{assumption}

\begin{assumption}[Effective-span summability for Main Theorem 2]
\label{ass:s-effective-span}
For the conditional state label under consideration,
\begin{equation}
\sum_{k\ge1} b_k\tau_k<\infty,
\qquad
\sum_{k\ge1}b_kJ_k(w)^2<\infty,
\end{equation}
where $b_k$ are the squared Hermite coefficients, $\tau_k=\int_0^\infty\rho(u)^k\dd u$, and $J_k(w)=\int_\R r_w(t)^k\dd t$.
\end{assumption}

The paper focuses on nonnegative correlations to state the boundary theorem cleanly.  The exact risk identity itself permits negative correlations whenever $C_g$ is evaluated on the corresponding interval.

\section{Gaussian and Hermite Preliminaries}
Let $H_k$ denote the probabilists' Hermite polynomials, normalized by
\begin{equation}
\E\{H_j(U)H_k(U)\}=k!\,\Ind\{j=k\},\qquad U\sim\cN(0,1).
\end{equation}
For centered $g\in L^2(\phi)$,
\begin{equation}
g(z)-\E g(U)=\sum_{k\ge1}\frac{a_k(g)}{k!}H_k(z),
\qquad
a_k(g)=\E[(g(U)-\E g(U))H_k(U)].
\label{eq:s-hermite}
\end{equation}
If $(U,V_r)$ are standard bivariate normal with correlation $r$, Mehler's identity gives
\begin{equation}
C_g(r)=\Cov\{g(U),g(V_r)\}
=\sum_{k\ge1}\frac{a_k(g)^2}{k!}r^k.
\label{eq:s-cg}
\end{equation}

For the threshold function $g_a(z)=\Ind\{z>a\}$, integration by parts yields, for $k\ge1$,
\begin{align}
a_k(a)
&=\int_a^\infty H_k(z)\phi(z)\dd z\\
&=\phi(a)H_{k-1}(a),
\label{eq:s-ind-coef}
\end{align}
because $H_k(z)\phi(z)=-\{H_{k-1}(z)\phi(z)\}'$.  Thus
\begin{equation}
b_k(a):=\frac{a_k(a)^2}{k!}
=\phi(a)^2\frac{H_{k-1}(a)^2}{k!}.
\label{eq:s-bk}
\end{equation}

A second identity is used for the trait component.  If $M,X$ are independent standard normals, then
\begin{equation}
\E\left[H_k\{\sqrt\alpha M+\sqrt{1-\alpha}X\}\mid M\right]
=\alpha^{k/2}H_k(M).
\label{eq:s-cond-hermite}
\end{equation}
It follows either from the generating function of $H_k$ or from the Gaussian Ornstein--Uhlenbeck semigroup.

\section{Posterior Predictor Used in Validation}
The main paper proves the exact protocol-risk identity. Here we record the computational form used to verify it. Let
\[
k_\pi(t)=\Cov\{Z(t),Y_\pi\},\quad
q_\pi(s,t)=k_\pi(s)^\top\Var(Y_\pi)^{-1}k_\pi(t).
\]
For $g_c(z)=\Ind\{z>c\}$, Gaussian conditioning gives
\begin{equation}
Z(t)\mid Y_\pi\sim\cN\{m_\pi(t),1-q_\pi(t,t)\},\quad
m_\pi(t)=k_\pi(t)^\top\Var(Y_\pi)^{-1}Y_\pi.
\end{equation}
The corresponding exact risk is
\begin{equation}
R_\pi^*=\frac1{T^2}\int_0^T\!\int_0^T
\left[C_g\{r_\alpha(s-t)\}-C_g\{q_\pi(s,t)\}\right]\dd s\dd t.
\label{eq:s-risk}
\end{equation}
Hence the Bayes predictor is
\begin{equation}
\widehat\Theta_{c,T}^{\,*}
=\frac1T\int_0^T\PhiBar\!\left(
\frac{c-m_\pi(t)}{\sqrt{1-q_\pi(t,t)}}
\right)\dd t.
\end{equation}
At $c=0$, $C_g(r)=\arcsin(r)/(2\pi)$.

\section{Fixed-Support Refinement}
The following lemma records the exact sense in which more segments are useful but do not become additional time.

\begin{lemma}[Fixed-support refinement supporting the main-paper discussion]
\label{lem:s-fixed-support}
Let $\cF_M$ be an increasing sequence of sigma-fields generated by progressively finer observations on a fixed temporal set $W$, and suppose $\cF_M\uparrow\cF_W$.  For every $\Theta\in L^2$,
\begin{equation}
\E\Var(\Theta\mid\cF_M)
\downarrow
\E\Var(\Theta\mid\cF_W).
\end{equation}
\end{lemma}

\begin{proof}
The martingale $\E(\Theta\mid\cF_M)$ converges to $\E(\Theta\mid\cF_W)$ in $L^2$ by the martingale convergence theorem.  Since
\begin{equation}
\E\Var(\Theta\mid\cF_M)=\E(\Theta^2)-\E[\E(\Theta\mid\cF_M)^2],
\end{equation}
the result follows.  Monotonicity is the projection property of conditional expectation.
\end{proof}

The limiting risk can be zero in exceptional analytically determined processes; the paper's statement explicitly concerns processes for which unobserved temporal support retains innovations relevant to the target.  Additional segments can reduce measurement noise, including noise that obscures the trait, but cannot change $W$.

\section{Verification of Main Theorem 1: Trait--State Decomposition}
\begin{theorem}[Trait--state decomposition; corresponds to Main Theorem 1]
\label{thm:s-trait-state}
Assume $0\le\rho\le1$ and Assumption~\ref{ass:s-summability}. Then
\begin{equation}
\Var(\Theta_{g,T})=C_g(\alpha)+\frac{A_{\rm state}(g)}{T}
+o(T^{-1}),
\end{equation}
where
\begin{equation}
A_{\rm state}(g)=2\int_0^\infty
\left[C_g\{\alpha+(1-\alpha)\rho(u)\}-C_g(\alpha)\right]\dd u,
\end{equation}
equivalently given by Eq.~\eqref{eq:s-binomial}, and
$C_g(\alpha)=\Var[\E\{g(Z(t))\mid M\}]$.
\end{theorem}

\begin{proof}
By stationarity,
\begin{equation}
\Var(\Theta_{g,T})=\frac{2}{T^2}\int_0^T(T-u)C_g\{r_\alpha(u)\}\dd u.
\label{eq:s-stationary}
\end{equation}
Because $2T^{-2}\int_0^T(T-u)\dd u=1$, Eq.~\eqref{eq:s-stationary} can be written exactly as
\begin{equation}
\Var(\Theta_{g,T})=C_g(\alpha)+\frac{2}{T^2}\int_0^T(T-u)\Delta_g(u)\dd u,
\label{eq:s-exactdecomp}
\end{equation}
where $\Delta_g(u)=C_g\{\alpha+(1-\alpha)\rho(u)\}-C_g(\alpha)$.

The main-paper asymptotic follows directly from the exact decomposition:
\begin{align}
T\{\Var(\Theta_{g,T})-C_g(\alpha)\}
&=2\int_0^T\left(1-\frac{u}{T}\right)\Delta_g(u)\dd u.
\end{align}
Dominated convergence gives $A_{\rm state}(g)=2\int_0^\infty\Delta_g(u)\dd u$. If the first absolute moment of $\Delta_g$ is finite, adding and subtracting the infinite integral bounds the remainder by an $O(T^{-2})$ term plus $T^{-1}\int_T^\infty|\Delta_g(u)|\dd u$. This finite-$T$ bound is the appendix-level verification not needed for the leading statement in the main paper.

Insert Eq.~\eqref{eq:s-cg} and expand binomially:
\begin{align}
A_{\rm state}(g)
&=2\sum_{k\ge1}\frac{a_k(g)^2}{k!}
\int_0^\infty\left[\{\alpha+(1-\alpha)\rho(u)\}^k-\alpha^k\right]\dd u\\
&=2\sum_{k\ge1}\frac{a_k(g)^2}{k!}
\sum_{j=1}^k\binom{k}{j}\alpha^{k-j}(1-\alpha)^j\tau_j.
\label{eq:s-binomial}
\end{align}
Tonelli's theorem applies directly when $0\le\rho\le1$.

For completeness, applying Eq.~\eqref{eq:s-cond-hermite} term-by-term to Eq.~\eqref{eq:s-hermite} verifies the trait identity
\begin{equation}
\Var\left[\E\{g(Z(t))\mid M\}\right]
=\sum_{k\ge1}\frac{a_k(g)^2}{k!}\alpha^k=C_g(\alpha).
\end{equation}
\end{proof}

\subsection{One-Snapshot Explainability}
Let $Y=Z(T/2)+\varepsilon$ with $\varepsilon\sim\cN(0,\nu^2)$. The protocol-explained covariance is
\begin{equation}
q_Y(s,t)=\frac{r_\alpha(|s-T/2|)r_\alpha(|t-T/2|)}{1+\nu^2}.
\end{equation}
If $\alpha>0$, Ces\`aro convergence and the exact-risk identity give
\begin{equation}
\lim_{T\to\infty}\Var\{\E(\Theta_{g,T}\mid Y)\}
=C_g\left(\frac{\alpha^2}{1+\nu^2}\right),
\end{equation}
and therefore
\begin{equation}
\lim_{T\to\infty}\cI_Y
=\frac{C_g\{\alpha^2/(1+\nu^2)\}}{C_g(\alpha)}.
\end{equation}
For the mean and occupation labels, if $\alpha=0$ and $\rho$ is integrable, the explained variance is $O(T^{-2})$ while the label variance is $O(T^{-1})$, yielding $\cI_Y=O(T^{-1})$.

\begin{corollary}[Trait-channel value of occasions and within-occasion segments; corresponds to Main Corollary 1]
\label{cor:s-trait-channel}
Consider the mean label as $T\to\infty$. Suppose $D$ occasions are separated enough that their state terms are independent, and each occasion averages $M$ independent raw segments with measurement-noise variance $\sigma_\varepsilon^2$ per segment. Then the protocol explainability for the limiting trait target $\sqrt\alpha M_i$ is
\begin{equation}
\cI_{\rm trait}(D,M)=
\frac{\alpha}
{\alpha+(1-\alpha)/D+\sigma_\varepsilon^2/(DM)}.
\label{eq:s-traitdm}
\end{equation}
\end{corollary}

\begin{proof}
The average across all observations can be written
\begin{equation}
\bar Y=\sqrt\alpha M_i+\sqrt{1-\alpha}\,\bar X_D+\bar\varepsilon,
\end{equation}
with independent components and variances $\alpha$, $(1-\alpha)/D$, and $\sigma_\varepsilon^2/(DM)$. The Gaussian regression coefficient of determination for predicting $\sqrt\alpha M_i$ from $\bar Y$ is $\Cov(\sqrt\alpha M_i,\bar Y)^2/[\Var(\sqrt\alpha M_i)\Var(\bar Y)]$, which simplifies to Eq.~\eqref{eq:s-traitdm}.
\end{proof}

Same-time replication has $D=1$: as $M\to\infty$ it removes measurement noise but leaves transient state variance. Increasing temporally separated occasions additionally averages the state component, so the two replication axes are statistically distinct even for the trait channel.
After standardization, a two-occasion test--retest design at a lag with negligible transient correlation identifies $\alpha$ from cross-occasion covariance and $\sigma_\varepsilon^2$ from observed variance (or within-occasion segments). Thus this trait ceiling is available from ordinary repeated-measurement data; it does not require estimation of $\rho$ or $\tau_{k\ge2}$.

\section{Verification of Main Theorem 2: Effective Temporal Span}
Condition on $M=m$. For the occupation label, the trait-conditioned threshold for the unit state process is
\begin{equation}
a=\frac{c-\sqrt\alpha m}{\sqrt{1-\alpha}}.
\end{equation}
Let $W_w$ be a standardized noisy average of $X$ over a fixed window of length $w$ centered at $t_0=T/2$, and write
\begin{equation}
r_w(u)=\Corr\{X(t_0+u),W_w\},\qquad
J_k(w)=\int_\R r_w(u)^k\dd u.
\end{equation}

\begin{theorem}[Task-dependent state-effective span; corresponds to Main Theorem 2]
\label{thm:s-effective-span}
Assume the window remains interior as $T\to\infty$, $r_w^k\in L^1(\R)$ for every Hermite order carrying nonzero weight, and the effective-span summability assumption holds. Then:
\begin{enumerate}[leftmargin=1.5em,itemsep=2pt]
\item For the mean state label,
\begin{equation}
\cI_{\rm state}(w)=\frac{\ell_{\rm mean}(w)}{T}+o(T^{-1}),
\qquad
\ell_{\rm mean}(w)=\frac{2\tau_1}{\eta(w)+\nu^2},
\label{eq:s-ellmean}
\end{equation}
where
\begin{equation}
\eta(w)=\frac{2}{w^2}\int_0^w(w-u)\rho(u)\dd u.
\end{equation}
\item For the occupation state label $\Theta_{a,T}=T^{-1}\int_0^T\Ind\{X(t)>a\}\dd t$,
\begin{equation}
\cI_{\rm state}(w\mid a)=\frac{\ell_a(w)}{T}+o(T^{-1}),
\label{eq:s-effective}
\end{equation}
with
\begin{equation}
\ell_a(w)=
\frac{\displaystyle\sum_{k\ge1}\frac{H_{k-1}(a)^2}{k!}J_k(w)^2}
{\displaystyle2\sum_{k\ge1}\frac{H_{k-1}(a)^2}{k!}\tau_k}.
\label{eq:s-ella}
\end{equation}
\end{enumerate}
\end{theorem}

\begin{proof}
\medskip\noindent\textit{Mean label.}\ 
For a unit state process,
\begin{equation}
\eta(w)=\Var\left\{\frac1w\int_{-w/2}^{w/2}X(u)\dd u\right\}
=\frac{2}{w^2}\int_0^w(w-u)\rho(u)\dd u.
\end{equation}
If additive window noise has variance $\nu^2$, stationarity and Fubini's theorem give
\begin{equation}
J_1(w)=\int_\R r_w(u)\dd u
=\frac{2\tau_1}{\sqrt{\eta(w)+\nu^2}}.
\label{eq:s-J1}
\end{equation}
The long-horizon state-label variance is $2\tau_1/T+o(T^{-1})$. The variance explained by $W_w$ equals
\begin{equation}
\frac1{T^2}\left\{\int_0^T r_w(t-t_0)\dd t\right\}^2.
\end{equation}
Because $t_0=T/2$ and $r_w\in L^1(\R)$,
\begin{equation}
\int_0^T r_w(t-t_0)\dd t
=\int_{-T/2}^{T/2}r_w(u)\dd u
=J_1(w)+o(1).
\end{equation}
Dividing the explained variance by the state-label variance and using Eq.~\eqref{eq:s-J1} proves Eq.~\eqref{eq:s-ellmean}.

\medskip\noindent\textit{Occupation label.}\ 
Using Eq.~\eqref{eq:s-ind-coef},
\begin{equation}
\Var(\Theta_{a,T})=\frac{A_a}{T}+o(T^{-1}),
\qquad
A_a=2\sum_{k\ge1}b_k(a)\tau_k.
\label{eq:s-Aa}
\end{equation}
For jointly normal $X(t)$ and $W_w$,
\begin{equation}
\E\{H_k(X(t))\mid W_w\}=r_w(t-t_0)^kH_k(W_w).
\end{equation}
Orthogonality therefore gives
\begin{equation}
\Var\{\E(\Theta_{a,T}\mid W_w)\}
=\frac1{T^2}\sum_{k\ge1}b_k(a)
\left\{\int_0^T r_w(t-t_0)^k\dd t\right\}^2.
\end{equation}
For each fixed $k$, absolute integrability and the interior placement imply
\begin{equation}
\int_0^T r_w(t-t_0)^k\dd t
=\int_{-T/2}^{T/2}r_w(u)^k\dd u
=J_k(w)+o(1).
\end{equation}
The summability assumption permits passage of this limit through the Hermite series, yielding
\begin{equation}
\Var\{\E(\Theta_{a,T}\mid W_w)\}
=\frac{B_a(w)}{T^2}+o(T^{-2}),
\qquad
B_a(w)=\sum_{k\ge1}b_k(a)J_k(w)^2.
\end{equation}
Dividing by Eq.~\eqref{eq:s-Aa} proves Eq.~\eqref{eq:s-effective}. Substituting $b_k(a)=\phi(a)^2H_{k-1}(a)^2/k!$ yields Eq.~\eqref{eq:s-ella}; the factor $\phi(a)^2$ cancels exactly.
\end{proof}

For $\rho(u)=e^{-u/\tau}$ and $\nu^2=0$,
\begin{equation}
\eta(w)=\frac{2\tau\{w-\tau(1-e^{-w/\tau})\}}{w^2},
\qquad
\ell_{\rm mean}(w)=\frac{w^2}{w-\tau(1-e^{-w/\tau})}.
\end{equation}
Taylor expansion gives $\ell_{\rm mean}(w)=2\tau+(2/3)w+O(w^2/\tau)$ as $w/\tau\downarrow0$, and $\ell_{\rm mean}(w)=w+\tau+O(\tau^2/w)$ as $w/\tau\to\infty$.

\begin{corollary}[Equal raw-segment budget in the state channel; corresponds to Main Corollary 2]
\label{cor:s-equal-budget}
Write $\ell_g(w;\nu^2,m)$ for the trait-conditioned effective span when an averaged window has noise variance $\nu^2$. Allocate $N$ independent raw measurements either to one fixed window, averaged to noise variance $\sigma_\varepsilon^2/N$, or to $N$ mutually separated windows, each with noise variance $\sigma_\varepsilon^2$. Under the sparse long-horizon conditions of the effective-span theorem above,
\begin{align}
\cI_{\rm same}(N\mid m)
&=\frac{\ell_g(w;\sigma_\varepsilon^2/N,m)}{T}+o(T^{-1})
\longrightarrow \frac{\ell_g(w;0,m)}{T},\label{eq:s-samebudget}\\
\cI_{\rm dispersed}(N\mid m)
&=\frac{N\ell_g(w;\sigma_\varepsilon^2,m)}{T}+o(N/T),
\label{eq:s-dispersedbudget}
\end{align}
until the first-order additivity approximation approaches saturation.
\end{corollary}

\begin{proof}
For the same-time allocation, averaging $N$ conditionally independent measurements changes only the window-noise variance from $\sigma_\varepsilon^2$ to $\sigma_\varepsilon^2/N$; Theorem~\ref{thm:s-effective-span} then gives Eq.~\eqref{eq:s-samebudget}, and continuity of the posterior projection in the noise variance gives the limit. For mutually separated windows, the off-diagonal window covariances and the cross terms in the explained state variance are negligible in the sparse regime. Each window contributes $B_g(w,m)/T^2+o(T^{-2})$, while the state-label variance is $A_g(m)/T+o(T^{-1})$. Summing the $N$ contributions gives Eq.~\eqref{eq:s-dispersedbudget}.
\end{proof}

For a mean label and point-like windows these expressions reduce to
\begin{equation}
\frac{2\tau_1}{T(1+\sigma_\varepsilon^2/N)}
\quad\text{and}\quad
\frac{2N\tau_1}{T(1+\sigma_\varepsilon^2)},
\end{equation}
respectively. Thus increasing $M$ purchases precision on fixed support, whereas increasing $D$ can purchase additional state support.

\section{Verification of Main Theorem 3: Boundary Localization}
Define
\begin{equation}
G_a(r)=\Cov\{\Ind(U>a),\Ind(V_r>a)\}.
\end{equation}
Plackett's identity gives
\begin{equation}
\frac{\partial}{\partial r}\Pr(U>a,V_r>a)
=\frac{\exp\{-a^2/(1+r)\}}{2\pi\sqrt{1-r^2}}.
\end{equation}
Since $G_a(0)=0$,
\begin{equation}
G_a(r)=\int_0^r\frac{\exp\{-a^2/(1+s)\}}{2\pi\sqrt{1-s^2}}\dd s.
\label{eq:s-plackett}
\end{equation}

\begin{theorem}[Boundary localization; corresponds to Main Theorem 3]
\label{thm:s-boundary}
Assume $\rho(u)\ge0$. Then
$A_a=2\int_0^\infty G_a\{\rho(u)\}\dd u$ is even and nonincreasing
in $|a|$. It is strictly decreasing in $|a|$ whenever $\rho$ is positive
on a set of nonzero measure, and hence is maximized at $a=0$.
\end{theorem}

\begin{proof}
For each fixed $s\in[0,1)$, the integrand in Eq.~\eqref{eq:s-plackett} is even in $a$ and strictly decreases with $|a|$. Integrating first in $s$ and then in $u$ gives the claimed properties of $A_a$; strictness holds whenever $\rho$ is positive on a set of nonzero measure.
\end{proof}

Equation~\eqref{eq:s-ella} is also even because $H_k(-a)^2=H_k(a)^2$.  Under locally uniform convergence it is differentiable and $\ell_a'(w)|_{a=0}=0$.  No general monotonicity is claimed for $\ell_a(w)$: the remaining threshold dependence is through the relative reweighting of Hermite orders.  The paper's numerical comparison shows that it is substantially flatter than $A_a$ for the investigated OU and Mat\'ern windows.

\section{Ornstein--Uhlenbeck Worked Example}
For $\rho(u)=e^{-u/\tau}$,
\begin{equation}
\tau_k=\int_0^\infty e^{-ku/\tau}\dd u=\frac{\tau}{k}.
\end{equation}
Using $r=e^{-u/\tau}$,
\begin{equation}
A_a=2\tau\int_0^1\frac{G_a(r)}{r}\dd r
=2\tau\phi(a)^2\sum_{k\ge1}\frac{H_{k-1}(a)^2}{k!\,k}.
\end{equation}
For $a=0$,
\begin{equation}
G_0(r)=\frac{\arcsin(r)}{2\pi}.
\end{equation}
Therefore
\begin{align}
A_0
&=\frac{\tau}{\pi}\int_0^1\frac{\arcsin(r)}{r}\dd r\\
&=\frac{\tau\log 2}{2}.
\end{align}
The last integral follows by the substitution $r=\sin x$ and the standard integral $\int_0^{\pi/2}x\cot x\dd x=(\pi/2)\log2$.

\section{Sparse Multi-Window Consequence}
Suppose $D$ equal windows are mutually separated so that all cross-window correlations are $o(1)$ in the asymptotic regime, and suppose $D\ell_g(w)/T=o(1)$.  Then the explained state variances add to first order:
\begin{equation}
\cI_{\rm state}(\pi_D)=\frac{D\ell_g(w)}{T}+o(D/T).
\end{equation}
Under per-window cost $c_0+c_1w$ and total per-object budget $B$, this gives the leading efficiency criterion
\begin{equation}
\max_w\frac{\ell_g(w)}{c_0+c_1w}.
\end{equation}
This is only a sparse-regime consequence.  Non-sparse placement must use the full $q_\pi$ in Eq.~\eqref{eq:s-risk} and is closely related to existing excursion-set sequential-design problems.

\section{Simulation Protocol and Additional Results}
The trait--state verification uses 50 repetitions and 2,000 independent objects per scenario. The equal-segment-budget experiment uses 30 repetitions and 1,500 objects. OU paths are generated on grids by the exact transition
\begin{equation}
X(t+\Delta t)=e^{-\Delta t/\tau}X(t)+\sqrt{1-e^{-2\Delta t/\tau}}\,\zeta,
\qquad \zeta\sim\cN(0,1).
\end{equation}
The final experiments were run on a Mac Studio with an Apple M2 Ultra CPU
(24 cores) and 192 GB RAM under macOS 26.5.2, using Python 3.14.4,
NumPy 2.4.4, SciPy 1.17.0, pandas 3.0.0, and Matplotlib 3.10.8; no GPU
was used.
For each object, the continuous occupation proportion is approximated on the fine grid.  A noisy point observation at $T/2$ is generated with variance $0.2$.  The predictor is the exact posterior mean of the discretized occupation proportion, not a fitted regression model.

\subsection{Reported Quantities}
For every scenario, the analysis computes:
\begin{itemize}[leftmargin=1.4em,itemsep=1pt]
\item exact finite-$T$ label variance;
\item exact one-snapshot explainability;
\item exact Bayes MSE $R^*=\Var(\Theta)(1-\cI)$;
\item Monte Carlo label variance, explained variance, explainability, and MSE;
\item Monte Carlo means and 95\% half-widths across repetitions;
\item equal-budget $D$--$M$ curves comparing repeated same-time segments with dispersed occasions.
\end{itemize}
The theory curves for effective spans are evaluated by numerical quadrature of $\tau_k$ and $J_k(w)$ with stable normalized-Hermite recurrences.  The OU boundary coefficient is also evaluated from its independent Plackett integral, providing a cross-check of the Hermite implementation.

\subsection{Equal-Segment-Budget Experiment}
The state-only experiment sets $\alpha=0$, $c=0$, $T/\tau=20$, and raw segment-noise variance one. For each total budget $N\in\{1,2,4,8,16,32,64\}$, the same-time protocol uses one midpoint occasion with averaged noise variance $1/N$, while the coverage protocol uses $N$ evenly spaced occasions with variance one each. Exact explainability is computed from Eq.~\eqref{eq:s-risk} using the arcsine transform; Monte Carlo uses the exact discrete OU posterior. At $N=64$, exact explainability is $0.0970$ for same-time replication and $0.8083$ for dispersed occasions.

\subsection{Numerical Cross-Checks}
The final values in the main-paper table are checked by two independent
calculations: direct evaluation of the exact formulas and Monte Carlo
estimation under the discretized posterior. The displayed half-widths summarize
variation across the independent repetitions described above.

\section{Interpretive Boundaries}
\paragraph{Trait channel.}
Purely cross-sectional observations do not identify $\alpha$, but ordinary
test--retest data do under the model. The trait ceiling needs only $\alpha$ and
$\sigma_\varepsilon^2$ and therefore does not share the state channel's dense
short-lag calibration requirement.

\paragraph{State channel.}
Occupation-time state variance and effective span depend on $\tau_2,\tau_3,\ldots$, not only $\tau_1$. Two kernels matched at $\tau_1$ can therefore agree on the leading long-horizon mean-label coefficient while disagreeing on occupation-label information. State-channel analysis requires a dense short-lag calibration subset, external longitudinal data, or a defensible parametric family.

\paragraph{Average effective span.}
If one defines a population state-effective span as $\E_m B_g(w,m)/\E_m A_g(m)$, this is a ratio of expectations, not an expectation of individual ratios.  It is weighted toward individuals with larger state-driven label variance, which for occupation labels are those near the threshold.  It should not be interpreted as the span of a typical object.

\endgroup
\end{document}